\documentclass[11pt,reqno]{amsart}

\usepackage{microtype}
\usepackage{graphicx}
\usepackage{booktabs}
\usepackage{tikz}
\usepackage{tikz-cd}
\usepackage{quiver}
\usepackage{siunitx}

\newcommand\crule[3][black]{\textcolor{#1}{\rule{#2}{#3}}}
\definecolor{Blue}{RGB}{55,126,184}
\definecolor{Orange}{RGB}{255,127,0}

\usepackage{hyperref}

\usepackage{enumerate}
\usepackage{amsmath,amsthm,amssymb,bbold}
\usepackage{float}
\usepackage{easybmat}
\usepackage{mathtools} 

\usepackage[capitalize,noabbrev]{cleveref}

\theoremstyle{plain}
\newtheorem{theorem}{Theorem}[section]
\newtheorem{proposition}[theorem]{Proposition}
\newtheorem{lemma}[theorem]{Lemma}

\theoremstyle{definition}

\theoremstyle{remark}

\newcommand{\bitem}{\begin{itemize}}
\newcommand{\eitem}{\end{itemize}}
\newcommand{\mc}[1]{\mathcal{#1}}

\newcommand{\II}{\mathbb{I}}
\newcommand{\N}{\mathbb{N}}
\newcommand{\R}{\mathbb{R}}

\newcommand{\EE}{\mathbb{E}}

\newcommand{\bpm}{\begin{pmatrix}}
\newcommand{\epm}{\end{pmatrix}}
\newcommand{\bsm}{\left(\begin{smallmatrix}}
\newcommand{\esm}{\end{smallmatrix}\right)}
\newcommand{\T}{\top}

\newcommand{\ol}[1]{\overline{#1}}
\newcommand{\wt}{\widetilde}
\newcommand{\wh}{\widehat}
\newcommand{\la}{\langle}
\newcommand{\ra}{\rangle}

\newcommand{\vphi}{\varphi}

\newcommand{\dd}{\textup{d}}
\newcommand{\eins}{\mathbb{1}}

\DeclareMathOperator{\Diag}{Diag}

\DeclareMathOperator{\vvec}{vec}

\DeclareMathOperator{\KL}{KL}

\newcommand{\expv}[1]{{\exp_{#1}^\mathfrak{v}}}

\newcommand{\Rv}[1]{{R_{#1}^\mathfrak{v}}}
\newcommand{\Projv}{{\Pi_0^\mathfrak{v}}}
\newcommand{\barycenterv}{{\eins_\mc{W}^\mathfrak{v}}}
\newcommand{\features}{{F}}
\newcommand{\dataspace}{{\mc{D}}}
\newcommand{\prior}{{\pi}}
\newcommand{\posterior}{{\mu}}
\newcommand{\pushdistr}{{\eta}}
\newcommand{\classmarginal}{{\pushdistr^{(1)}}}
\newcommand{\classindexset}{{\mc{I}}}
\newcommand{\tangentspace}{\mc{T}_{0}}
\newcommand{\samplesize}{{m}}
\newcommand{\momentcoord}{{\wh{\mathfrak{m}}}}
\newcommand{\moment}{{\mathfrak{m}}}
\newcommand{\momentmarg}{{\wt{\mathfrak{m}}}}
\newcommand{\datadistr}{{\mathfrak{D}}}
\newcommand{\risk}{{\mathfrak{L}}}
\newcommand{\emprisk}{{\mathfrak{L}_\samplesize}}
\newcommand{\qmcsamples}{{q}}
\newcommand{\softmax}{{\exp_{\eins_\mc{W}}}}
\newcommand{\Hmeasures}{{\mathcal{P}}}
\DeclareMathOperator{\ldaf}{\psi}
\DeclareMathOperator{\expm}{expm}

\title[Self-Certifying Classification by Linearized Deep Assignment]{Self-Certifying Classification by Linearized Deep Assignment}
\author[B.~Boll, A.~Zeilmann, S.~Petra, C.~Schn\"{o}rr]{Bastian Boll, Alexander Zeilmann, Stefania Petra, \\Christoph Schn\"{o}rr}
\address{Image and Pattern Analysis Group, Heidelberg University, Germany}
\email{bastian.boll@iwr.uni-heidelberg.de}
\urladdr{\url{https://ipa.math.uni-heidelberg.de}}
\date{}

\begin{document}

\begin{abstract}
We propose a novel class of deep stochastic predictors for classifying metric data on graphs within the PAC-Bayes risk certification paradigm. Classifiers are realized as linearly parametrized deep assignment flows with random initial conditions. Building on the recent PAC-Bayes literature and data-dependent priors, this approach enables (i) to use risk bounds as training objectives for learning posterior distributions on the hypothesis space and (ii) to compute tight out-of-sample risk certificates of randomized classifiers more efficiently than related work.
Comparison with empirical test set errors illustrates the performance and practicality of this self-certifying classification method.
\end{abstract}

\maketitle
\tableofcontents


\section{Introduction}\label{sec:introduction}

\subsection{Overview, Related Work}\label{sec:overview}

Self-certified learning is the task of using the entirety of available data to find a good model and to simultaneously certify its performance on unseen data from the same underlying distribution.
This is opposed to the classic two-stage paradigm in machine learning which first finds a model by using part of the data and subsequently estimates its generalization on held-out test data.

Because the true distribution of data is typically unknown, self-certified learning relies on upper-bounding model \emph{risk} through statistical learning theory.
Recently, the PAC-Bayes (Probably Approximately Correct) paradigm \cite{Catoni:2007aa,Guedj:2019tu} has attracted much attention due to the recent demonstration of tight risk bounds for deep stochastic neural networks in \cite{Dziugaite:2018wc}. The authors exploit a PAC-Bayes risk bound by, firstly, training a prior through empirical risk minimization and, secondly, by training a Gibbs posterior distribution.
Building on this, recent work \cite{Perez:2021} evaluated various \textit{relaxed PAC-Bayes-kl inequalities}\footnote{The lowercase kl refers to the relative entropy of two Bernoulli distributions.} \cite{Langford:2001}, including a new one. They showed that non-vacuous risk certificates can be determined numerically which are informative of the out-of-sample error, and that using relaxed upper bounds of the risk for training enables the use of the whole data set for both learning a predictor and certifying its risk.

Similar to \cite{Perez:2021}, our approach is to find a PAC-Bayes posterior distribution by optimizing the PAC-Bayes-$\lambda$ inequality introduced by \cite{Thiemann:2017}. This relaxed PAC-Bayes-kl inequality was shown to be quasiconvex in the parameter that trades off empirical error against model complexity in terms of KL divergence, which is convenient for optimization.

A key component of PAC-Bayes bounds is the empirical risk of stochastic classifiers. In the context of deep learning, such classifiers may be obtained by randomizing neural network weights which typically leads to analytically intractable empirical risk.
\cite{Langford:2002} therefore suggest using an upper bound via Monte-Carlo sampling, which holds with high probability and still achieves PAC risk certification with modified probability of correctness.
In order to train stochastic classifiers by optimizing PAC inequalities with differentiable surrogate loss, the gradient of empirical risk can similarly be estimated stochastically. \cite{Perez:2021} choose the pathwise gradient estimator \cite{Price:1958} and call the resulting framework PAC-Bayes with Backprop, reminiscent of the Bayes-by-Backprop paradigm \cite{Blundell:2015}.

Here, we propose a way to achieve computational tractability of empirical risk in PAC-Bayes without the need for stochastic estimators. Key is the construction of a specific hypothesis class which separates stochasticity from feature extraction by building on certain geometric neural ODEs called \emph{assignment flows} \cite{Astroem2017}. After suitable parameterization and linearization, the uncertainty quantification approach proposed in \cite{Gonzalez-Alvarado:2021vn} allows to push forward intrinsic normal distributions of initial assignment states in closed form, which can be leveraged to build deep stochastic classifiers with tractable empirical risk. 

\subsection{Contribution}\label{sec:contribution}

We adopt the PAC-Bayes-$\lambda$ inequality \cite{Thiemann:2017} to work out a two-stage method as in \cite{Dziugaite:2018wc,Perez:2021} for evaluating relaxed PAC-Bayes-kl bounds and achieve favorable computational properties compared to prior work.

To this end, we propose a generalized, deep classification variant of S-assignment flows \cite{Savarino:2021wt} and compute the corresponding pushforward distribution in closed form, building on \cite{Gonzalez-Alvarado:2021vn}.
Separating stochasticity from feature extraction, we use the computationally tractable pushforward distribution to transform the empirical risk of stochastic classifiers. This requires to compute the pushforward \emph{only once} and subsequently perform very cheap sampling of a transformed integrand in Monte-Carlo methods.
Finally, we show that for not too large numbers $c \approx 10$ of classes, much more efficient deterministic Quasi-Monte-Carlo integration \cite{Dick:2013aa} can replace Monte-Carlo estimation when evaluating risk certificates. We also show that this deterministic method commutes with backpropagation, i.e., the true gradient is well approximated by backpropagation.

Altogether, this enables us to evaluate the empirical risk and its gradient very efficiently, and to optimize the risk bound provided by the PAC-Bayes-$\lambda$ inequality with respect to both the posterior and the trade-off between empirical loss and deviation from a data-dependent prior distribution. As a consequence, the \textit{linearized deep assignment flow approach} to classification (Figure~\ref{fig:architecture}) becomes a \textit{self-certifying learning method}.
We verify its performance and the tightness of risk certificates by a comparison to the empirical test error.

\begin{figure}[ht]
	\centering
	\includegraphics[width=.8\columnwidth]{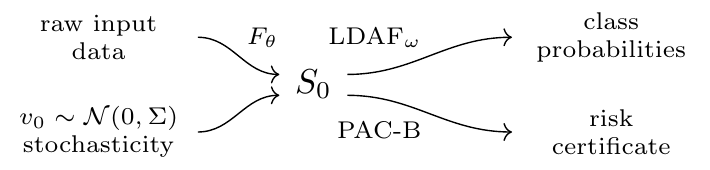}
	\caption{Proposed architecture for self-certifying stochastic classifiers. Features $\features_\theta$ are extracted from input data and define the initial state $S_0\in\mc{W}$ of assignment dynamics. The initialization is randomized on the tangent space $\tangentspace$ leading to a stochastic classifier $\text{LDAF}_\omega$ after forward integration. A PAC-Bayes bound (PAC-B) is employed for computing non-vacuous certificates (upper bound) of the classifier's risk.}\label{fig:architecture}
\end{figure}

\clearpage

\section{Background}\label{sec:background}

\subsection{(S-)Assignment Flows}\label{sec:background_af}

The \textit{assignment flow approach} \cite{Astroem2017,Schnorr2019aa} denotes a class of dynamical systems for analyzing metric data on a graph $G=(V,E)$, $|V| = n$, that is derived in a straightforward way: represent local decisions as point (`state') on a task-specific statistical manifold and perform \textit{contextual structured} decisions by the interaction of these states over the underlying graph. As for the \textit{classification} task considered here, the statistical manifold is the probability simplex equipped with the Fisher-Rao metric of information geometry \cite{ay2017information}, and the interaction corresponds to \textit{geometric} state averaging derived from the affine e-connection. Adopting the
parametrization of \cite{Savarino:2021wt}, the resulting dynamical system reads
\begin{equation}\label{eq:s_flow}
    \dot S(t) = R_{S(t)}[\Omega S(t)],\qquad S(0) = S_0,
\end{equation}
where $S(t)\in\mc{W}\subset\R_{++}^{n\times c}$ comprises the state at each vertex $i\in V$ as row vector $S_{i}(t) \in \R_{++}^{c}$, where $2\leq c\in\N$ denotes the number of classes. The underlying geometry always restricts $S(t)$ to the set of row-stochastic matrices with full support called \textit{assignment manifold} $\mc{W}$, with trivial tangent bundle $T\mc{W}=\mc{W}\times\tangentspace$, where
\begin{equation}\label{eq:def-mcT0}
\mc{T}_{0}=\{V\in \R^{n\times c}\colon \la \eins_{c},V_{i}\ra=0,\; i\in V\}.
\end{equation}
The right-hand side of \eqref{eq:s_flow} constitutes a proper vector field on $\mc{W}$ that is parametrized by the matrix $\Omega$ and mapped to $T\mc{W}$ by the state-dependent non-orthogonal projection mapping that acts row-wise by
\begin{equation}\label{eq:def-RS}
R_{S}[\Omega S]=\big(R_{S_{1}}(\Omega S)_{1},\dotsc,R_{S_{n}}(\Omega S)_{n}\big)^{\T},
\end{equation}
where $S_{i}, (\Omega S)_{i}$ denote the $i$th row vectors and
\begin{equation}\label{eq:def-RSi}
R_{S_{i}}=\Diag(S_{i})-S_{i}S_{i}^{\T}\in \R^{c\times c},\quad i\in V.
\end{equation}
Each row of \eqref{eq:s_flow} defines a \textit{replicator equation} \cite{Hofbauer:2003aa} that are \textit{coupled} over the graph through the right-hand side and interact by geometric numerical integration \cite{Zeilmann:2020aa} of the assignment flow $S(t)$. Under mild conditions on $\Omega$, $\lim_{t\to\infty} S(t)$ \text{converges} to hard label assignments and is \textit{stable} against data perturbations \cite{Zern:2020aa}.

As a consequence, \eqref{eq:s_flow} may be seen as a particular system of neural ODEs \cite{Chen:2018ab} that represent the layers of a deep network by time-discrete geometric numerical integration of the flow.
In Section \ref{sec:daf}, we adopt a `deep structured' parametrization of \eqref{eq:s_flow} and restrict ourselves to a \textit{linearization} of the resulting large-scale dynamical system.

Our main contribution is to show that a state-of-the-art PAC Bayes bound \cite{Thiemann:2017} can be evaluated both rigorously and efficiently for the corresponding hypothesis class of \textit{linearized deep assignment flows} (Section \ref{sec:risk_certification}) and lead to \textit{non-vacuous quantitative} risk certificates (Section \ref{sec:benchmarks}).

\subsection{PAC-Bayes Risk Certification}\label{sec:background_pac}

Consider stochastic classifiers, i.e., distributions $\posterior$ over a hypothesis space $\mc{H}$, elements of which are functions $\phi_{\theta} \colon \dataspace \to \tangentspace$.
Suppose $\mc{H}$ is parameterized by $\theta\in\Theta$ and identify distributions $\rho$ over $\mc{H}$ with distributions over the parameter space $\Theta$.
For given loss function $\ell\colon \tangentspace\to \R$ and a generally unknown data distribution $\datadistr$ over $\dataspace \times \tangentspace$, the goal of learning stochastic classifiers is to find $\posterior$ such that the expected risk
\begin{equation}\label{eq:def_risk}
    \EE_{\theta\sim\posterior} [\risk(\theta)] := \EE_{\theta\sim\posterior} \big[ \EE_{(x,y)\sim \datadistr}\ell(\phi_\theta(x), y)\big]
\end{equation}
is minimized. Since $\datadistr$ is unknown, the true risk $\risk(\theta)$ is difficult to estimate. A tractable related quantity is the \emph{empirical risk}
\begin{equation}\label{eq:def_emprisk}
    \emprisk(\theta) := \frac{1}{\samplesize} \sum_{k\in [\samplesize]} \ell(\phi_\theta(x_k), y_k)
\end{equation}
where $(x_k, y_k)$ denote $\samplesize$ i.i.d.~samples drawn from $\datadistr$.
PAC-Bayesian theory \cite{Catoni:2007aa,Guedj:2019tu} considers a distribution $\posterior$ called \emph{PAC-Bayes posterior} which depends on the sample as well as a reference distribution $\prior$ called \emph{PAC-Bayes prior} which does not depend on the \emph{same} sample. A goal is to construct tight, high-confidence bounds on \eqref{eq:def_risk} which only depend on tractable quantities such as \eqref{eq:def_emprisk}.
In our analysis, we use the following state-of-the-art bound.
\begin{theorem}[\textbf{PAC-Bayes-$\lambda$ Inequality }\cite{Thiemann:2017}]\label{theorem:TIWS_bound}
For any $\epsilon > 0$ and any $\lambda\in (0,2)$, it holds with probability at least $1-\epsilon$ for all posterior distributions $\posterior$ over parameters $\theta$ simultaneously
\begin{equation}\label{eq::TIWS_bound}
    \EE_{\theta\sim\posterior} [\risk(\theta)] \leq \frac{\EE_{\theta\sim\posterior}[\emprisk(\theta)]}{1-\frac{\lambda}{2}} + \frac{\KL(\posterior\colon \pi) + \log \frac{2\sqrt{\samplesize}}{\epsilon}}{\samplesize\lambda (1-\frac{\lambda}{2})},
\end{equation}
where $\samplesize$ denotes the size of an i.i.d.~sample set.
\end{theorem}
Regarding the evaluation of the right-hand side, key issues are the definition of prior and posterior distributions $\pi, \mu$ over the hypothesis space and the accurate and efficient computation of the \textit{expected} empirical risk $\EE_{\theta\sim\posterior}[\emprisk(\theta)]$, which typically is a hard task in practice. We deal with these issues in Sections \ref{sec:hypothesis_space}, \ref{sec:data_dependent_prior} and \ref{sec:empirical_risk}, \ref{sec:QMC}, respectively.

\section{Deep Assignment Flows}\label{sec:daf}

\subsection{Deep S-Flows}\label{sec:generalized_sflow}
Motivated by the use of coupled replicator dynamics in game theory \cite{Madeo:2014}, we generalize S-flows \eqref{eq:s_flow} by enabling additional interaction on the label space.
To shorten notation, vectorize $S, V, R_{S}$ given by \eqref{eq:s_flow}--\eqref{eq:def-RSi} and the orthogonal projection $\Pi_{0}\colon\R^{|V|\times c}\to\mc{T}_{0}$ according to
\begin{gather}
    s := \vvec(S),\qquad v := \vvec(V)\\
    \Projv v := \vvec(\Pi_0V),\qquad \Rv{s}v := \vvec(R_{S}V)\\
    \expv{s}v := \vvec(\exp_{S}V),
\end{gather}
where $\exp_{S}(V)=(\exp_{S_{1}}(V_{1}),\dotsc,\exp_{S_{1}}(V_{1}))\in\mc{W}$ and $\exp_{S_{i}}(V_{i})=\frac{S_{i}e^{V_{i}}}{\la S_{i},e^{V_{i}}\ra}$ with componentwise multiplication in the numerator.
In vectorized notation, the S-flow dynamics \eqref{eq:s_flow} read
\begin{equation}\label{eq:s_flow_vectorized}
	\dot s(t) = \Rv{s(t)}(\Omega \otimes \II_c)s(t),\qquad s(0) = s_0 = \vvec(S_0)\ .
\end{equation}
We now generalize by breaking up the Kronecker product structure of $\Omega \otimes \II_c$. Re-using the symbol $\Omega$ to denote a matrix
\begin{equation}\label{eq:wt_omega}
    \Omega\in \R^{N\times N}, \qquad N = cn
\end{equation}
we define the \emph{deep assignment flow} (DAF) in vectorized form as
\begin{equation}\label{eq:daf_flow}
    \dot s(t) = \Rv{s(t)}\Omega s(t),\qquad s(0) = s_0\ .
\end{equation}
This class of dynamics is more general than \eqref{eq:s_flow} while remaining amenable to lifting and linearization with minimal modifications to other assignment flows \cite{Zeilmann:2020aa,Boll:2021vb}. Concerning the PAC-Bayes risk certification, we observe that \eqref{eq:daf_flow} typically leads to better generalization and more gain between posterior and prior as compared to \eqref{eq:s_flow}.

\subsection{Classification by Deep Assignment}\label{sec:classification_daf}
Unlike typical assignment flow approaches, our aim is not to perform image labeling (i.e., segmentation) but classification. To this end, we choose the underlying graph to be relatively small ($n = 50$ nodes) and densely connected with learned symmetric matrix $\Omega \in \R^{N\times N}$ (cf. \eqref{eq:wt_omega}). We also designate a single node to carry class probabilities. Through the dynamics \eqref{eq:daf_flow}, the state of this node will evolve towards an integer assignment, i.e., a class decision.
By convention, let the classification node be the node with index $1$ and let
\begin{equation}\label{eq:class_node_index}
	\classindexset = [c] := \{1,\dotsc,c\}
\end{equation}
denote the set of indices such that $s_\classindexset = S_1$ contains classification probabilities.
Further, denote the relative interior of a single probability simplex with $c$ corners by $\mc{S}_c$.

\subsection{Linearizing Deep Assignment Flows}\label{sec:daf_linearization}
We parameterize DAFs in the tangent space of $\mc{W}$ at $S_0$ by $s(t) = \expv{s_0}(v(t))$ where $v(t)$ follows
\begin{equation}\label{eq:tangent_space_daf}
    \dot v(t) = \Projv \Omega \expv{s_0}(v(t)),\qquad v(0) = v_0 = 0
\end{equation}
and compute using $d\exp_{S_{0;i}}(V_{0;i})[U]=R_{S_{0;i}}[U]$
\begin{align}
    \Projv \Omega\expv{s_0}(v_0) &= \Projv \Omega s_0\\
    \dd (\Projv \circ\Omega\circ\expv{s_0})(v_0)[u] &= \Projv \Omega \dd (\expv{s_0})(v_0)[u] = \Projv \Omega \Rv{s_0} u
\end{align}
which yields the \textit{linearized} DAF
\begin{equation}\label{eq:linearized_daf}
    \dot v(t) = \underbrace{\Projv \Omega \Rv{s_0}}_{=: A}v(t) + \underbrace{\Projv \Omega s_0}_{=: v_D}
\end{equation}
The solution in closed form reads
\begin{equation}\label{eq:linearized_daf_solution}
    v(t) = t\varphi(tA)v_D = t\varphi\left(t\Projv \Omega \Rv{s_0}\right)\Projv \Omega s_0
\end{equation}
and Krylov methods for evaluating the analytical matrix-valued function $\vphi(z)=\frac{e^{z}-1}{z}$ as well as respective gradient approximations computed in \cite{Zeilmann:2021wt} apply without modification.
Note that even though $\varphi(tA)$ acts linearly on $v_D$ in \eqref{eq:linearized_daf_solution}, LDAF dynamics are much more capable than a learned linear map $\tangentspace\to T_0\mc{S}_c$. This is due to the fact that $A = \Projv \Omega \Rv{s_0}$ depends on $s_0$ so \textit{each} input datum is transformed by a \textit{different} linear operator.

\section{Risk Certification of Stochastic LDAF Classifiers}\label{sec:risk_certification}

We consider PAC-Bayes risk certificates which bound the expected risk of a stochastic classifier (see Section~\ref{sec:background_pac}). The evaluation of such a certificate requires evaluation of expected empirical risk which presents a computational challenge. To mitigate this, one may use Monte-Carlo methods to upper-bound the expected empirical risk with high probability as proposed in \cite{Langford:2002}. Here, we propose instead a strategic choice of hypothesis class and shape of stochastic classifiers which allows to directly compute the expected empirical risk efficiently and precisely while also allowing for the use of deep feature extractors.

\subsection{LDAF Hypothesis Space}\label{sec:hypothesis_space}

We define the hypothesis space $\mc{H}$ of classifiers $\phi$ built by composing a feature extractor with LDAF dynamics \eqref{eq:linearized_daf} up to time $T > 0$.
To this end, fix a small, densely connected graph. We use $n = 50$, $c = 10$ in our experiments.
This defines an associated assignment manifold $\mc{W}$ and tangent space $\tangentspace$ on which a linear operator $\Omega$ specifies DAF dynamics \eqref{eq:tangent_space_daf} with linearization \eqref{eq:linearized_daf}. We assume $\Omega$ is symmetric and denote the vector of learnable parameters defining $\Omega$ by $\omega$.
For a given data point $x$ in some vector space $\dataspace$ such as the space of RGB images, a corresponding initial point $s_0\in\mc{W}$ is computed by extracting features using a neural network $\features_\vartheta\colon \dataspace\to \tangentspace$ with parameters $\vartheta$ and setting $s_0 = \softmax (\features_\vartheta(x))$.
Following linearization of the DAF vector field, we take the initialization $v_0\in \tangentspace$ as additional parameters. Forward integration up to time $T$ gives a state $s(T) = \exp_{s_0}(v(T))\in \mc{W}$ which contains class probabilities $S(T)_1\in\mc{S}_c$ at the classification node.
We collect the described sequence of operations on $\mc{W}$ into a function $\ldaf_{\omega,v_0}$ and call
\begin{equation}\label{eq:ldaf_pac_hypothesis_class}
    \mc{H} = \{\phi\colon \R^d\to \mc{S}_c\;|\; \phi = \ldaf_{\omega,v_0} \circ \;\softmax\circ \features_\vartheta\}
\end{equation}
the hypothesis class of \emph{LDAF classifiers}. Measures on $\mc{H}$ are identified with measures on the parameter space $\ol{\mc{H}}$ which contains triples $\theta = (\vartheta, \omega, v_0)$.

Denote by $\Hmeasures$ the class of probability measures $\posterior$ on $\mc{H}$ with shape
\begin{equation}\label{eq:measure_product_shape}
    \posterior = \delta_{\vartheta} \times \delta_{\omega} \times \mc{N}(0,\Sigma_0)
\end{equation}
where $\mc{N}(0,\Sigma_0)$ denotes an intrinsic normal distribution on $\tangentspace$ (cf.~chapter 3 in \cite{Rue:2005aa}) centered at $0$.
Due to the structure of $\tangentspace$ as a linear subspace of $\R^{N}$, covariance matrices $\Sigma_0\in \R^{N\times N}$ are positive semi-definite.
Each measure in $\Hmeasures$ corresponds to a stochastic LDAF classifier which operates by taking an independent sample from $\posterior$ for each datum. In order to compute PAC-Bayes risk certificates, we need to compute the expected empirical risk of stochastic classifiers as well as their complexity with respect to a reference distribution. Suppose the reference distribution $\prior$ (PAC-Bayes prior) also has shape \eqref{eq:measure_product_shape}. Further, $\omega$ and $\vartheta$ are fixed and only the distribution of $v_0$ differs between $\prior$ and $\posterior$ (PAC-Bayes posterior). The following lemma asserts that within the constructed setting, model complexity in terms of relative entropy is well-defined and computationally feasible.
\begin{lemma}\label{lem:well_defined_kl}
Fix the distributions $\posterior = \delta_{\vartheta} \times \delta_{\omega} \times \mc{N}(0,\Sigma_0)$ and $\prior = \delta_{\vartheta} \times \delta_{\omega} \times \mc{N}(0,\wt{\Sigma}_0)$ in $\Hmeasures$. Then $\posterior$ is absolutely continuous with respect to $\prior$ ($\posterior \ll \prior$) and
\begin{equation}\label{eq:KL_divergence}
    \KL [\posterior\colon \prior] = \KL[\mc{N}(0,\Sigma_0)\colon \mc{N}(0,\wt{\Sigma}_0)]
\end{equation}
\end{lemma}
\begin{proof}
Let $A$ be a measurable subset of $\ol{\mc{H}}$ with $\prior(A) = 0$. Then
\begin{equation}
    \delta_{\vartheta}(A_1) \delta_{\omega}(A_2) \mc{N}(0,\wt{\Sigma}_0)(A_3) = 0
\end{equation}
for projections $A_1$, $A_2$, $A_3$ of $A$ onto the respective coordinates. Therefore, at least one of the factors needs to vanish. If either of the first two vanishes, this directly implies $\posterior(A) = 0$.
In addition, $\mc{N}(0,\wt{\Sigma}_0)(A_3) = 0$ implies $\mc{N}(0,\Sigma_0)(A_3) = 0$ and thus $\posterior(A) = 0$ because both normal distributions define equivalent measures.
It follows $\posterior \ll \prior$. Because the three factors in $\posterior$ resp. $\prior$ are independent, the relative entropy decomposes as
\begin{equation}
\KL [\posterior\colon \prior] =
        \KL[\mc{N}(0,\Sigma_0)\colon \mc{N}(0,\wt{\Sigma}_0)]\\
        +\underbrace{\KL[\delta_{\vartheta}\colon \delta_{\vartheta}]}_{= 0}
        +\underbrace{\KL[\delta_{\omega}\colon \delta_{\omega}]}_{= 0}  
\end{equation}
\end{proof}

\subsection{Data-Dependent Prior}\label{sec:data_dependent_prior}

Recently, the use of data for finding a good prior $\prior$ has been identified as critical for obtaining sharp generalization bounds. This development was sparked by non-vacuous risk bounds for neural networks achieved by \cite{Dziugaite:2018wc}. Unlike this work, we do not make use of differential privacy to account for sharing data between prior and posterior. Instead, we forego potentially more efficient use of data in favor of simplicity by splitting the available dataset into a \emph{training} and a \emph{validation} set. The training set is used to compute a PAC-Bayes prior distribution $\prior$ via empirical risk minimization. The validation set is subsequently used to fine-tune the PAC-Bayes posterior distribution $\posterior$ by minimizing a risk bound for a differentiable surrogate loss starting from $\prior$. In addition, the validation set is also used to evaluate the final classification risk certificate. In order to achieve the setting assumed in Lemma~\ref{lem:well_defined_kl}, we only vary the distribution of $v_0$ when optimizing $\posterior$ and keep the feature extraction parameters $\vartheta$ as well as the fitness parameters $\omega$ fixed.

\subsection{Computing the Expected Empirical Risk}\label{sec:empirical_risk}

We now aim to leverage the analytical tractability of LDAF forward integration to efficiently compute the empirical risk of stochastic classifiers in $\Hmeasures$. This can be done irrespective of feature extraction because stochasticity only pertains to the LDAF initialization $v_0$. Key to the construction is the ability to push forward a multivariate normal distribution on $\tangentspace$ under LDAF dynamics in closed-form. This amounts to an extension of the uncertainty quantification approach \cite{Gonzalez-Alvarado:2021vn} to the deep flows considered here.

\begin{proposition}[\textbf{LDAF Pushforward}]\label{prop:ldaf_pushforward}
Consider the LDAF dynamics \eqref{eq:linearized_daf} and let $v(0)\sim \mc{N}(0, \Sigma_0)$. Then $v(t)$ follows the multivariate normal distribution $\pushdistr(t) = \mc{N}(\moment(t), \Sigma(t))$ for every $t > 0$ with moments
\begin{subequations}\label{eq:pushforward_moments}
\begin{align}
    \moment(t) &= t\varphi(tA)b,\\
    \Sigma(t) &= \expm(tA)\Sigma_0\expm(tA)^\top
\end{align}
\end{subequations}
\end{proposition}
\begin{proof}
For $v_0 \neq 0$, the closed form solution \eqref{eq:linearized_daf_solution} is modified to
\begin{equation}
    v(t) = \expm (tA)v_0 + t\varphi(tA)b.
\end{equation}
We see that for fixed $t>0$, $v(0)$ is mapped to $v(t)$ by an affine transformation. Therefore, $v(t)$ still follows a multivariate normal distribution.
Denote its moments by
\begin{subequations}
\begin{align}
    \moment(t) &= \EE [v(t)],\\
    \Sigma(t) &= \EE [(v(t)-\moment(t))(v(t)-\moment(t))^\top].
\end{align}
\end{subequations}
One readily computes
\begin{equation}
    \dot \moment(t) = \EE [\dot v(t)]
        = \EE [Av(t)+b]
        = Am(t) + b
\end{equation}
which has the closed form solution
\begin{equation}
    \moment(t) = t\varphi(tA)b
\end{equation}
because $\moment(0) = \EE [v(0)] = 0$. A straightforward computation shows that
\begin{equation}
    \dot \Sigma(t) = \EE [ \dot v(t)(v(t)-\moment(t))^\top + (v(t)-\moment(t))\dot v(t)^\top ]
\end{equation}
with $\Sigma(0) = \Sigma_0$. Inserting the shape of \eqref{eq:linearized_daf} now gives
\begin{align}\label{eq:pushforward_covariance}
    \dot \Sigma(t) 
        &= A\Sigma(t) + \Sigma(t) A^\top
\end{align}
with closed form solution
\begin{equation}\label{eqbb:cov_closed_form}
    \Sigma(t) = \expm(tA)\Sigma_0\expm(tA)^\top
\end{equation}
analogous to the computation in \cite{Gonzalez-Alvarado:2021vn}.
\end{proof}

The full covariance matrix $\Sigma(t)\in \R^{N\times N}$ in \eqref{eq:pushforward_covariance} is quite large ($N = nc$ as in \eqref{eq:wt_omega}) and expensive to compute. However, for the purpose of classification, we only need the marginal $\classmarginal(T)$ of $\pushdistr(T)$ for the classification node (the first node by convention). Because $\pushdistr(T)$ is a normal distribution, the moments of $\classmarginal(T)$ are the subvectors of \eqref{eq:pushforward_moments} built from all entries with indices $\classindexset$ (cf.~\eqref{eq:class_node_index}).
We may now leverage the available closed form \eqref{eq:pushforward_moments} to transform the empirical risk of stochastic LDAF classifiers, which is the main technical contribution of this paper.

\begin{theorem}[\textbf{LDAF Expected Empirical Risk}]\label{theorem:ldaf_empirical_risk}
Fix (linear) coordinates of $T_0\mc{S}_c$ by choosing the columns of
\begin{equation}\label{eq:T0_coordinates}
    P := \bpm \II_{c-1}\\ -\eins_{c-1}^\top\epm\in \R^{c\times (c-1)}
\end{equation}
as basis vectors. For a given data sample $\{(x_k, y_k)\}_{k\in [\samplesize]}$ and loss function $\ell\colon T_0\mc{S}_c\times [c]\to \R$, the stochastic classifier with distribution $\posterior = \delta_{\vartheta}\times\delta_{\Omega}\times\mc{N}(0,\Sigma_0)$ on the hypothesis class $\mc{H}$ has expected empirical risk $\EE_{v_0 \sim \posterior}[\emprisk(v_0)]$ given by
\begin{equation}\label{eq:ldaf_emp_risk}
    \frac{1}{\samplesize} \sum_{k\in [\samplesize]} \int_{\R^{c-1}} \ell(Pz + \features_\vartheta(x_k), y_k) \rho_{\momentcoord(x_k), \wh\Sigma(x_k)}(z)\dd z.
\end{equation}
Here, $\rho$ denotes the density of a multivariate normal distribution with the indicated moments
\begin{subequations}
\begin{align}
    \momentcoord(x_k) &= \moment(T)_{\classindexset\setminus \{c\}}\\
    \wh \Sigma(x_k) &= \Sigma(T)_{\classindexset\setminus \{c\},\classindexset\setminus \{c\}}
\end{align}
\end{subequations}
which are subvectors resp.\ submatrices of \eqref{eq:pushforward_moments} for each input datum derived from the marginal distribution $\classmarginal(T)$ of $\pushdistr(T)$ for the classification node. The last index $c$ is omitted due to the shape of basis \eqref{eq:T0_coordinates}.
\end{theorem}
\begin{proof}
Note that $A$ and $b$ in \eqref{eq:pushforward_moments} depend on the initial assignment state $s_0 = \expv{\barycenterv}(\features_\vartheta(x))$ according to \eqref{eq:linearized_daf} which clarifies the data dependence of $\momentcoord = \momentcoord(x_k)$.
Because $T_0\mc{S}_c$ is a vector space, any proper multivariate normal distribution in $T_0\mc{S}_c$ corresponds to an improper intrinsic multivariate normal distribution in (linear) coordinates such as in the basis \eqref{eq:T0_coordinates}.
The particular choice of basis \eqref{eq:T0_coordinates} is such that mean an covariance are entries of the moments $\momentmarg = \moment(T)_\classindexset$, $\wt\Sigma = B B^\top$ of $\classmarginal(T)$ because
\begin{subequations}
\begin{align}
    \momentmarg &= \EE [v(T)_\classindexset] = \bpm \momentcoord\\ -\la \eins_{c}, \momentcoord\ra\epm\\
    \wt \Sigma &= \EE[(v(T)_\classindexset- \momentmarg)(v(T)_\classindexset-\momentmarg)^\top] \\
        &= P\EE[(v(T)_{\classindexset\setminus \{c\}}-\momentcoord)(v(T)_{\classindexset\setminus \{c\}}-\momentcoord)^\top]P^\top \\
        &= \bpm \wh\Sigma & -\wh \Sigma\eins\\ -\eins^\top\wh\Sigma & -\eins^\top \wh \Sigma \eins\epm.
\end{align}
\end{subequations}
We now transform the sought empirical risk by leveraging the shape of pushforward marginal under the LDAF dynamics.
\begin{subequations}
\begin{align}
    &\EE_{v_0 \sim \posterior}[\emprisk(v_0)] = \int \frac{1}{\samplesize}\sum_{k=1}^\samplesize \ell(\phi(x_k), y_k)\dd \posterior (\phi)\label{eq:pushforward_transform1}\\
        &\quad = \frac{1}{\samplesize}\sum_{k=1}^\samplesize \int_{T_0\mc{S}_c} \ell(v+\features_\vartheta(x_k)_\classindexset, y_k)\dd \classmarginal(v)\label{eq:pushforward_transform3}\\
        &\quad = \frac{1}{\samplesize}\sum_{k=1}^\samplesize \int_{\R^{c-1}} \ell(Pz+\features_\vartheta(x_k)_\classindexset, y_k)\rho_{\momentcoord(x_k), \wh\Sigma(x_k)}\dd z\label{eq:pushforward_transform4}
\end{align}
\end{subequations}
Here, \eqref{eq:pushforward_transform3} uses the marginal distribution $\classmarginal$ of the pushforward computed in Proposition~\ref{prop:ldaf_pushforward}. To obtain class probabilities, the tangent vector resulting from LDAF forward integration needs to be lifted at $s_0 = \exp_{\eins_\mc{W}}(\features_\vartheta(x_k))$. In \eqref{eq:pushforward_transform3}, we have expressed this by a shift in classification logits due to
\begin{subequations}
	\begin{align}
		\exp_{\eins_\mc{W}}^{-1}(\exp_{s_0}(v))
		&= \exp_{\eins_\mc{W}}^{-1}(\exp_{\exp_{\eins_\mc{W}}(v^0)}(v))\\
		&= \exp_{\eins_\mc{W}}^{-1}(\exp_{\eins_\mc{W}}(v^0+v))\\
		&= v^0+v
	\end{align}
\end{subequations}
with $v^0 = \exp_{\eins_\mc{W}}^{-1}(s_0) = \features_\vartheta(x_k)$.
\end{proof}

Let $\sqrt{\cdot}$ denote the matrix square root and define the matrix $B\in\R^{c\times N}$ of the first $c$ rows of $\expm(TA)\sqrt{\Sigma_0}$. Then $B B^\top$ is the covariance matrix of $\classmarginal(T)$ and rows
\begin{equation}
    B_i = \sqrt{\Sigma_0}\expm(TA^\top)e_i,\qquad i\in \mc{I} = [c]
\end{equation}
can be efficiently computed by approximating the matrix exponential action via Krylov subspace methods \cite{Saad1992,Hochbruck2010,Niesen:2012}.
Computing $\classmarginal(T)$ is therefore only roughly $c$ times more expensive than a single LDAF forward pass if the action of $\sqrt{\Sigma_0}$ is efficiently computable. We ensure this by the parametrization
\begin{equation}\label{eq:sigma0_shape}
	\sqrt{\Sigma_0} = (\II_n\otimes P)\big(\Diag(d)+qq^\top\big) \in \R^{N\times N}
\end{equation}
with $P$ given by \eqref{eq:T0_coordinates}.
This allows to compute the relative entropy \eqref{eq:KL_divergence} using the Sherman-Morrison formula and the matrix determinant lemma such that learning the parameters $d,q \in \R^{(c-1)n}$ by minimizing the PAC-Bayes bound \eqref{eq::TIWS_bound} is computationally efficient.
Here, we use the parameterization \eqref{eq:sigma0_shape} for both PAC-Bayes prior $\prior$ and posterior $\posterior$ (cf. Section~\ref{sec:benchmarks}).

\subsection{Numerical Integration}\label{sec:QMC}

Previous works on PAC-Bayes risk certification have commonly resorted to approximating the expected empirical risk by a Monte-Carlo (MC) method. This accounts for very high-dimensional domains of integration, but it is computationally expensive because it requires evaluation of the integrand at many sample points which entails a separate forward pass for every drawn sample.
Theorem~\ref{theorem:ldaf_empirical_risk} proposes a way to circumvent this problem by computing the pushforward distribution \textit{only once} (at roughly the cost of $c$ forward passes) and by subsequently performing very cheap sampling of the integrand. This amounts to large efficiency gains when using MC methods.

Even though the domain of integration in \eqref{eq:pushforward_transform4} is low-dimensional compared to, e.g., the parameter space considered in \cite{Perez:2021} (millions of neural network parameters), standard product cubature rules still are not able to compete with MC as the number of function evaluations for such methods grows exponentially with dimension.

However, Quasi-Monte-Carlo (QMC) methods can improve on MC in the case at hand by leveraging smoothness and moderate dimension. The rationale behind QMC methods is to choose a sequence of deterministic sample points which has lower discrepancy than the uniform random points used in MC. By the Koksma-Hlawka inequality \cite{Dick:2013aa}[Theorem 3.9], the error of approximate integration via an unweighted sample point mean is bounded by the Hardy-Krause variation of the integrand times the discrepancy of the sample points. Thus, sequences of low-discrepancy sample points, such as the Sobol sequence, asymptotically lead to more efficient integration than MC if the integrand has bounded Hardy-Krause variation. The rate of convergence was shown to further improve under stricter smoothness assumptions  \cite{Basu:2016}. In practice, QMC methods are observed to outperform MC particularly for moderate dimension and smooth integrands \cite{Morokoff:1995}.
We observe that relatively few (10K) sample points suffice to compute the empirical risk of stochastic LDAF classifiers with sufficient accuracy. This is not the case of MC as illustrated in Figure~\ref{fig:qmc_performance}.

\begin{figure}[ht]
\centering
\includegraphics[width=.7\columnwidth]{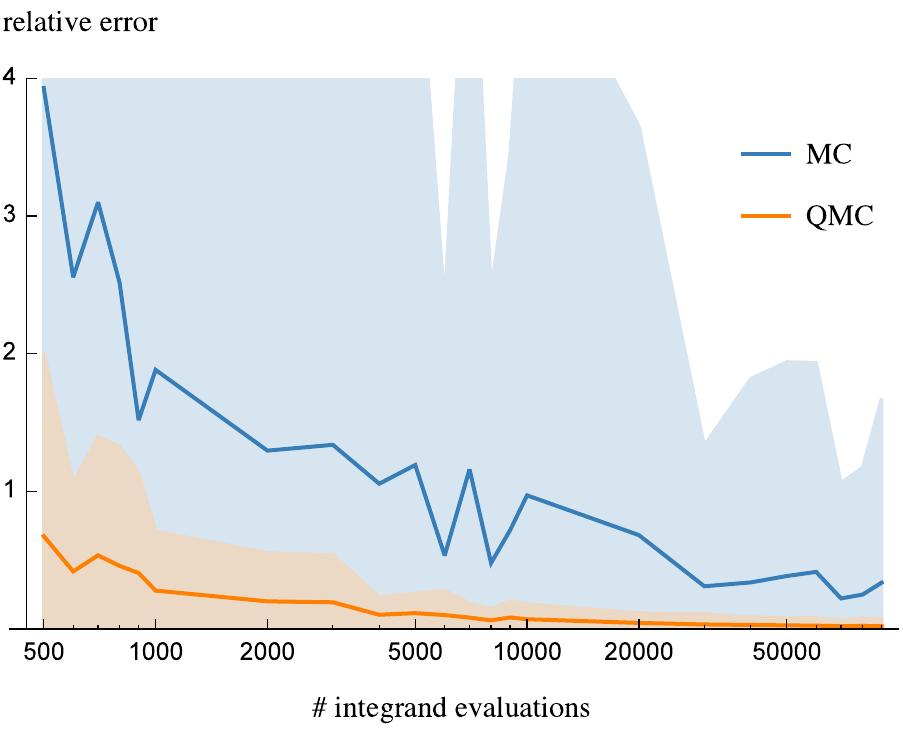}
\caption{Accuracy of QMC integration \crule[Orange]{1.75mm}{1.75mm}
	and MC integration \crule[Blue]{1.75mm}{1.75mm} for computing the expected
	empirical risk \eqref{eqbb:standardized_empirical_integral_last} with varying
	number of sample points. Error bands indicate standard deviation within a batch of 100 CIFAR-10 data points. Because the pushforward distribution of Theorem~\ref{theorem:ldaf_empirical_risk} is computationally tractable, sampling is very efficient and computing the reference solution by drawing 100M MC samples only takes minutes on a single GPU. In our proposed QMC method, we compute the pushforward distribution and subsequently perform 10K integrand evaluations at negligible computational cost.
}\label{fig:qmc_performance}
\end{figure}


To make QMC methods more easily applicable, we transform the integral in \eqref{eq:pushforward_transform4} by successive substitutions
\begin{subequations}\label{eqbb:standardized_empirical_integral}
\begin{align}
    \int_{\R^{c-1}} &\ell(Pz+\features_\vartheta(x_k)_\classindexset, y_k)\rho_{\momentcoord(x_k), \wh\Sigma(x_k)}\dd z\\
        &= \int_{[0,1]^{c-1}} \ell(PH\Phi^{-1}(z)+\moment(T)+\features_\vartheta(x_k)_\classindexset, y_k)\dd z\label{eqbb:standardized_empirical_integral_last}
\end{align}
\end{subequations}
where $\Phi$ is elementwise the cumulative distribution function of a standard normal distribution and $\wh\Sigma = HH^\top$ denotes Cholesky decomposition.
For moderate $c$ such as the classification problems considered in Section~\ref{sec:benchmarks} ($c = 10$), QMC integration allows to compute the empirical risk very precisely while using relatively few (10K) sample points.

In addition to gained efficiency, QMC integration has another distinct advantage over MC: it \emph{commutes with backpropagation}.
To see this, use the shorthand notation $f(z, \theta)$ for the integrand in \eqref{eqbb:standardized_empirical_integral_last} and assume the loss function $\ell$ is differentiable. QMC integration with deterministic Sobol points $\{z_k\}_{k\in [\qmcsamples]}$ followed by derivation reads
\begin{equation}\label{eq:integrate_then_diff}
	\frac{\partial}{\partial \theta_i} \int_{[0,1]^{c-1}} f(z, \theta)\dd z
		\approx \frac{\partial}{\partial \theta_i} \frac{1}{\qmcsamples} \sum_{k\in [\qmcsamples]} f(z_k, \theta)
\end{equation}
and exchanging summation and differentiation on the r.h.s.~yields
\begin{equation}\label{eq:diff_then_integrate}
		\frac{1}{\qmcsamples} \sum_{k\in [\qmcsamples]} \frac{\partial}{\partial \theta_i} f(z_k, \theta)\\
		\approx \int_{[0,1]^{c-1}} \frac{\partial}{\partial \theta_i} f(z_k, \theta)\dd z\ .
\end{equation}
Thus, if integration and differentiation can be exchanged then backpropagation through QMC integration gives the same result as approximating the exact derivative via QMC.

\section{Benchmarks and Discussion}\label{sec:benchmarks}

We now compare self-certifying stochastic classifiers trained by minimizing PAC-Bayes generalization bounds to deterministic classifiers evaluated on held-out test data on the CIFAR-10 and FashionMNIST datasets.

\subsection{Training Stochastic LDAF Classifiers}\label{sec:training}

Consider the image classification task on CIFAR-10 \cite{Krizhevsky:2009} and FashionMNIST \cite{Xiao:2017}.
Even though self-certified learning allows to use the entire dataset for training and judge generalization via risk certification, we still hold out the preassigned test sets for direct comparability with deterministic classifiers. The remaining data are split into a training set used for training priors as well as deterministic classifiers and a validation set ($\samplesize = 10$k) used for training posteriors and evaluating risk certificates.

As a baseline, we use a deterministic ResNet18 classifier \cite{He:2016} with a slight modification to account for small input size (see Appendix~\ref{sec:appendix_resnet}). We train on the described training split (with held-out validation and test data) for 200 epochs of stochastic gradient descent (weight decay $0.001$, momentum $0.9$, batch size $128$) with a cosine annealing learning rate schedule \cite{Loshchilov:2016} starting at learning rate $0.1$ and a light data augmentation regime as in \cite{Zagoruyko:2016}.

Stochastic classifiers $\prior$ and $\posterior$ are implemented as distributions with shape \eqref{eq:measure_product_shape} over the hypothesis space $\mc{H}$. The graph $G$ is chosen relatively small ($n=50$ nodes) and densely connected with symmetric matrix $\Omega$ as in \eqref{eq:wt_omega}. For feature extraction $\features$, we replace the classification head of the ResNet18 described above with a dense layer mapping to $\tangentspace$, i.e., dimension $N = nc = 50\cdot 10$.
Both PAC-Bayes prior $\prior$ and posterior $\posterior$ are implemented by randomizing LDAF initialization $v_0$ on the tangent space $\tangentspace$ according to a zero-mean multivariate normal distribution with covariance parameterized as in \eqref{eq:sigma0_shape}.

To train stochastic classifiers, we proceed in two steps. First, using the \emph{same training routine} as for ResNet18, we train a deterministic LDAF classifier on the training split. This defines the mean of stochastic classifiers in $\mc{H}$. For the PAC-Bayes prior $\prior$, we fix the covariance $\Sigma_0$ parameterized according to \eqref{eq:sigma0_shape} by drawing the entries of $d,p \in \R^{(c-1)n}$ from a univariate normal distribution centered at $0.1$ with variance $0.01$. Initializing $\posterior$ at $\prior$, we subsequently train $\posterior$ by minimizing the r.h.s. of the bound \eqref{eq::TIWS_bound}, alternating between optimization of $\posterior$ and $\lambda$. According to \cite{Thiemann:2017}, the PAC-Bayes-$\lambda$ bound \eqref{eq::TIWS_bound} is strongly quasiconvex as a function of $\lambda$ under mild conditions, which ensures convergence of alternating optimization. For each alternation, we optimize posteriors over 5 epochs of SGD (learning rate 0.1) on the validation set and find empirically that both $\lambda$ and $\posterior$ converge quickly ($< 10$ alternations).
The performances of stochastic classifiers as observed on the held-out test data as well as risk certificates computed on validation data are listed in Table~\ref{tab:benchmarks}.
We provide the code used to compute these values as supplementary material\footnote{Code can be found at \url{https://github.com/IPA-HD/ldaf_classification}}.

\begin{table}[]
\caption{Empirical performance of deterministic (top) and stochastic (middle) classifiers measured as 01 loss (error percentage) on the held-out test data of CIFAR-10 and FashionMNIST. PAC-Bayes risk certificates (bottom) bound the risk of stochastic classifiers with high probability $1-\epsilon$.
Tightness of risk certificates is on par with the state-of-the-art results (PBB = PAC-Bayes with Backprop) reported in \cite{Perez:2021} and our novel method offers improved computational efficiency due to tractable empirical risk.\\}
\label{tab:benchmarks}
\centering
\begin{tabular}{@{}lrr@{}}
\toprule
                              & CIFAR-10 & FashionMNIST   \\ \midrule
ResNet18                      &  \textbf{5.00} & \textbf{4.81}\\
LDAF Mean                     &  5.28      &  5.13   \\
\hline
Prior (ours)          &         {5.49} & 5.13 \\
Posterior (ours)      &  \textbf{5.31} & 5.12 \\
Posterior PBB         &        {14.75} & \\
\hline
Cert. (ours) $\epsilon = 0.01$ &  \textbf{6.36} & 6.07 \\
Cert. (ours) $\epsilon = 0.05$ &  \textbf{6.19} & 5.90 \\
Cert. PBB $\epsilon = 0.035$ & 16.67 & \\ \bottomrule
\end{tabular}
\end{table}

\subsection{Discussion and Conclusion}\label{sec:discussion}

As indicated in Table~\ref{tab:benchmarks}, the empirical performance of stochastic classifiers built in the proposed way is close to the performance of a (deterministic) ResNet18 used as a baseline. Notably, this is without altering the training regime.

The indicated error rate on CIFAR-10 is much lower than the one in \cite{Perez:2021}. We attribute this to ResNet18 being a stronger feature extractor than the 15-layer CNN used in \cite{Perez:2021}. By comparison with the stochastic prediction error rate, risk certificates are also very tight.
This is roughly on par with the results of \cite{Perez:2021} which improves on earlier work by \cite{Dziugaite:2018wc}.

Our contribution is not to find tighter risk certificates, but to provide a novel method that enables a more computationally efficient way to compute them. Because pushing forward an intrinsic normal distribution of initial LDAF assignment states is only about $c$ times more expensive than an LDAF forward pass, sampling the integrand in \eqref{eqbb:standardized_empirical_integral_last} is very cheap by comparison to earlier works, which require a separate forward pass per sample. This allows to, e.g., compute the reference solution used in Figure~\ref{fig:qmc_performance} in 14 minutes on a single GPU for 100M Monte-Carlo samples and batch size 100.

We use cross-entropy as a differentiable surrogate loss for training PAC-Bayes posteriors. This appears problematic because the bound \eqref{eq::TIWS_bound} only certifies risk w.r.t. bounded loss functions. \cite{Perez:2021} address this by modifying cross-entropy to obtain a closely related bounded loss function which is amenable to risk certification. We do not perform this modification and therefore do not obtain valid risk certificates for surrogate loss. However, for classification (01 loss) the bound \eqref{eq::TIWS_bound} holds for \emph{all} posterior distributions, regardless of how they have been computed. Therefore, using unbounded surrogate loss for training does not touch the validity of risk certificates for the bounded 01-loss reported in Table~\ref{tab:benchmarks}. Accordingly, no certificate for surrogate loss is reported.

A key component of the proposed approach are linearized deep assignment flows (LDAFs). We view them as uniquely suitable due to the combination of two factors. (1) The pushforward of normal distributions under LDAF dynamics has a closed form and efficient numerics exist to approximate its moments. (2) Unlike trivial maps which have the first property, the LDAF still has nontrivial representational power.

To illustrate this, attempt to replace the LDAF within the given framework by a learned dense linear operator. Aside from potentially introducing a large number of additional parameters, this effectively merely adds noise to the classification logits and only the stochastic classifier mean depends on data. Thus, essentially no improvement of the posterior over the prior is to be expected. By contrast, the low-rank numerics used to compute the pushforward under the LDAF reveal additional information about learnt parameters that we will exploit in future work.

\section*{Acknowledgements}

This work is funded by the Deutsche Forschungsgemeinschaft (DFG), grant SCHN 457/17-1, within the priority programme SPP 2298: “Theoretical Foundations of Deep Learning”.

This work is funded by the Deutsche Forschungsgemeinschaft (DFG) under Germany’s Excellence Strategy EXC-2181/1 - 390900948 (the Heidelberg STRUCTURES Excellence Cluster).

\bibliography{self_certified}
\bibliographystyle{amsalpha}

\newpage
\appendix
\onecolumn
\section{ResNet Training}\label{sec:appendix_resnet}

The ResNet18 architecture was proposed by \cite{He:2016} for classification on ImageNet which contains much larger images than the ones considered here. In order to account for small image dimensions (CIFAR: $32\times 32$ pixels, FashionMNIST: $28\times 28$ pixels), we swap the first convolution in ResNet18 (stride 2, kernel size 7) for a smaller one (stride 1, kernel size 3). We train with randomly initialized weights and do not change the training regime between CIFAR-10 and FashionMNIST. Our feature extractors and reference classifier implementation are based on the freely available PyTorch implementation \url{https://github.com/kuangliu/pytorch-cifar}.

\end{document}